\documentclass{article}

\usepackage{icml2020}

\usepackage{natbib}

\usepackage[utf8]{inputenc} 
\usepackage[T1]{fontenc}    
\usepackage{hyperref}       
\usepackage{url}            
\usepackage{booktabs}       
\usepackage{amsfonts}       
\usepackage{nicefrac}       
\usepackage{microtype}      

\usepackage{amsthm}
\newtheorem{theorem}{Theorem}
\theoremstyle{definition}
\newtheorem{definition}{Definition}[section]

\usepackage{titlesec}
\titlespacing{\section}{%
  0pt}{
  \baselineskip}{
  0em}
\titlespacing{\subsection}{%
  0pt}{
  \baselineskip}{
  0em}
\titlespacing{\paragraph}{%
  0pt}{
  0\baselineskip}{
  0.5em}

\usepackage[font=small,labelfont=bf]{caption}
\usepackage{subcaption}
\usepackage{array}
\newcolumntype{P}[1]{>{\centering\arraybackslash}p{#1}}
\usepackage{amsmath}
\usepackage{amsthm}
\usepackage[pdftex]{graphicx}
\usepackage{optidef}
\usepackage{amsopn}
\usepackage{enumitem}
\setitemize{noitemsep,topsep=0pt,parsep=0pt,partopsep=0pt}
\setdescription{noitemsep,topsep=0pt,parsep=0pt,partopsep=0pt}
\usepackage[margin,status=final]{fixme}
\fxusetheme{color}
\usepackage{tikz}
\usepackage{siunitx}
\usepackage{bm}
\usepackage{setspace}
\usepackage{cleveref}
\usepackage{xcolor}

\renewcommand{\vec}[1]{\bm{#1}}

\usepackage{braket}

\usepackage[acronym, smallcaps, nogroupskip, nonumberlist]{glossaries}
\glsdisablehyper
\newacronym{map}{MAP}{maximum a posteriori}
\newacronym{minlp}{MINLP}{mixed-integer nonlinear program}
\newacronym{minlo}{MINLO}{mixed-integer nonlinear optimization}
\newacronym{mio}{MIO}{mixed-integer optimization}
\newacronym{miqcp}{MIQCP}{mixed-integer quadratic constrained program}
\newacronym{miqp}{MIQP}{mixed-integer quadratic program}
\newacronym{nlp}{NLP}{nonlinear problem}
\newacronym{gmm}{GMM}{Gaussian mixture model}
\newacronym{abb}{$\alpha$BB}{$\alpha$-based Branch and Bound}
\newacronym{baron}{BARON}{Branch-and-Reduce Optimization Navigator}
\newacronym{gbd}{GBD}{generalized Benders' decomposition}
\newacronym{slsqp}{SLSQP}{Sequential Least-Squares Quadratic Programming}

\icmltitlerunning{MAP Clustering under the Gaussian Mixture Model via MINLO}

\def\isaccepted{true}

\begin{document}

\twocolumn[
    \icmltitle{MAP Clustering under the Gaussian Mixture Model via\\
    Mixed Integer Nonlinear Optimization}
    \textit{\icmlsetsymbol{equal}{*}
    \begin{icmlauthorlist}
    \icmlauthor{Patrick Flaherty}{um}
    \icmlauthor{Pitchaya Wiratchotisatian}{wpi}
    \icmlauthor{Ji~Ah Lee}{um}
    \icmlauthor{Zhou Tang}{um}
    \icmlauthor{Andrew C.~Trapp}{foise}
    \end{icmlauthorlist}
    \icmlaffiliation{um}{Department of Mathematics and Statistics, University of Massachusetts Amherst, USA}
    \icmlaffiliation{wpi}{Data Science Program, Worcester Polytechnic Institute, USA}
    \icmlaffiliation{foise}{Data Science Program and Foisie Business School, Worcester Polytechnic Institute, USA}
    \icmlcorrespondingauthor{Patrick Flaherty}{flaherty@math.umass.edu}
        \icmlkeywords{Combinatorial Optimization, Clustering}}
        \vskip 0.3in]


\printAffiliationsAndNotice{}  

\begin{abstract}
We present a global optimization approach for solving the maximum a-posteriori (MAP) clustering problem under the Gaussian mixture model. Our approach can accommodate side constraints and it preserves the combinatorial structure of the MAP clustering problem by formulating it as a mixed-integer nonlinear optimization problem (MINLP). We approximate the MINLP through a mixed-integer quadratic program (MIQP) transformation that improves computational aspects while guaranteeing $\epsilon$-global optimality. An important benefit of our approach is the explicit quantification of the degree of suboptimality, via the optimality gap, en route to finding the globally optimal MAP clustering. Numerical experiments comparing our method to other approaches show that our method finds a better solution than standard clustering methods. Finally, we cluster a real breast cancer gene expression data set incorporating intrinsic subtype information; the induced constraints substantially improve the computational performance and produce more coherent and biologically meaningful clusters.
\end{abstract}

\section{Introduction}
\label{sec:intro}

In the application of clustering models to real data there is often rich prior information that constrains the relationships among the samples, or the relationships between the samples and the parameters.
For example, in biological or clinical experiments, it may be known that two samples are technical replicates and should be assigned to the same cluster, or it may be known that the mean value for control samples is in a certain range. 
However, standard model-based clustering methods make it difficult to enforce such hard logical constraints and may fail to provide a globally optimal clustering.


Locally optimal and heuristic clustering methods lack optimality guarantees or bounds on solution quality that indicate the potential for further improvement. 
Even research on bounding the quality of heuristic solutions is scarce~\cite{cochran2018informs}.
Conversely, a rigorous optimization framework affords bounds that establish the degree of suboptimality for any feasible solution, and a certificate of global optimality thereby guaranteeing that no better clustering exists.

Recent work on developing global optimization methods for supervised learning problems has led to impressive improvements in the size of the problems that can be handled.
In linear regression, properties such as stability to outliers and robustness to predictor uncertainty have been represented as a mixed-integer quadratic program and solved for samples of size $n \sim 1,000$ \citep{bertsimas2015algorithmic, bertsimas2020scalable}.
Best subset selection in the context of regression is typically approximated as a convex problem with the $\ell_1$ norm penalty, but can now be solved exactly using the nonconvex $\ell_0$ penalty for thousands of data points \citep{bertsimas2016best}.
In unsupervised learning, \citet{bandi2019learning} use \gls{mio} to learn parameters for the Gaussian mixture model, showing that an optimization-based approach can outperform the EM algorithm. \citet{greenberg2019compact} use a trellis representation and dynamic programming to compute the partiton function exactly and thus the probability of all possible clusterings.
For these reasons we are motivated to develop a method for achieving a globally optimal solution for clustering under the Gaussian mixture model that allows for the incorporation of rich prior constraints.

\paragraph{Finite mixture model}
The probability density function of a finite mixture model is \(p( \vec{y} |
\vec{\theta}, \vec{\pi}) = \sum_{k=1}^{K} \pi_k p(\vec{y} | \vec{\theta}_k) \) where the observed data is \(\vec{y}\) and the parameter set is \(\vec{\phi} = \{\vec{\theta}, \vec{\pi}\}\).
The data is an \(n\text{-tuple}\) of \(d\text{-dimensional}\) random vectors \(\vec{y} = (\vec{y}_1^T, \ldots, \vec{y}_n^T)^T\) and the mixing proportion parameter \(\vec{\pi} = (\pi_1, \ldots, \pi_K)^T\) is constrained to the probability simplex \(\mathcal{P}_K = \Set{ \boldsymbol{p} \in \mathbb{R}^K | \boldsymbol{p} \succeq 0\ \textrm{and}\ \mathbf{1}^T \boldsymbol{p} = 1 }\).
When the component density, \(p(\vec{y} | \vec{\theta}_k)\), is a Gaussian
density function, \(p(\vec{y} | \vec{\phi})\) is a Gaussian mixture model with parameters \(\vec{\theta} = \left(\{\vec{\mu}_1, \vec{\Sigma}_1\}, \ldots, \{\vec{\mu}_K, \vec{\Sigma}_K\}\right)\).
Assuming independent, identically distributed (iid) samples, the Gaussian mixture model probability density function is
\(p(\vec{y} | \vec{\theta}, \vec{\pi}) = \prod_{i=1}^{n} \sum_{k=1}^K \pi_k\ p\left( \vec{y}_i | \vec{\mu}_k, \vec{\Sigma}_k \right)\).

\paragraph{Generative model}
A generative model for the Gaussian mixture density function is
\begin{equation}
    \label{eqn:gen-model}
    \begin{split}
    Z_i &\overset{\text{iid}}{\sim} \textrm{Categorical}(\vec{\pi})\quad \text{for}\ i=1, \ldots, n,\\
    Y_i| z_i, \vec{\theta} &\sim \textrm{Gaussian}(\vec{\mu}_{z_i}, \vec{\Sigma}_{z_i}),
  \end{split}
\end{equation}
where \(\vec{\mu} = (\vec{\mu}_1, \ldots, \vec{\mu}_K)\) and \(\vec{\Sigma} = (\vec{\Sigma}_1, \ldots, \vec{\Sigma}_K)\).
To generate data from the Gaussian mixture model, first draw \(z_i \in \{1, \ldots, K\}\) from a categorical distribution with parameter \(\vec{\pi}\).
Then, given \(z_i\), draw \(\vec{y}_i\) from the associated Gaussian component distribution \(p(\vec{y}_i | \vec{\theta}_{z_i})\).

\paragraph{\Gls{map} clustering}
The posterior distribution function for the generative Gaussian mixture model is
\begin{equation}
  p(\vec{z}, \vec{\theta}, \vec{\pi} | \vec{y}) = \frac{p(\vec{y} | \vec{\theta}, \vec{z}) p(\vec{z} | \vec{\pi}) p(\vec{\theta}, \vec{\pi})}{p(\vec{y})}.
\end{equation}
The posterior distribution requires the specification of a prior distribution
\(p(\vec{\theta}, \vec{\pi})\), and if \(p(\vec{\theta}, \vec{\pi}) \propto 1\),
then \gls{map} clustering is equivalent to maximum likelihood clustering.
The \gls{map} estimate for the component membership configuration can be
obtained by solving the following optimization problem:

\begin{maxi}
  {\vec{z}, \vec{\theta}, \vec{\pi}}
  {\log p(\vec{z}, \vec{\theta}, \vec{\pi} | \vec{y})}
  {}
  {}
  \addConstraint{z_{i}}{\in \{1, \ldots, K\} \; \forall i}{}
  \addConstraint{\vec{\pi}}{\in \mathcal{P}_K}{}.
\end{maxi}

Assuming iid sampling, the objective function comprises the following conditional density functions:
\small
\begin{align*}
  p(\vec{y}_i | \vec{\mu}, \vec{\Sigma}, \vec{z}_i) &= \prod_{k=1}^K \left[ (2 \pi)^{-m/2} | \vec{\Sigma}_k |^{-1/2} \right. \\
  &\left. \quad \cdot \exp \left( -\frac{1}{2} (\vec{y}_i - \vec{\mu}_k)^T \vec{\Sigma}^{-1}_k (\vec{y}_i - \vec{\mu}_k) \right) \right]^{z_{ik}}, &   \\
  p(\vec{z}_i | \vec{\pi}) &= \prod_{k=1}^K \left[ \pi_k \right]^{z_{ik}},\qquad p(\vec{\Sigma}, \vec{\pi}, \vec{\mu}) \propto 1,&
\end{align*}
\normalsize

where $z_i \in \{1, \ldots, K\}$ is recast using binary encoding.
To simplify the presentation, consider the case of one-dimensional data (\(d=1\)) and equivariant components
(\(\vec{\Sigma}_1 = \cdots = \vec{\Sigma}_K = \sigma^2\)).
The \gls{map} optimization problem can be written
\small
\begin{mini}|s|
  {\vec{z}, \vec{\mu}, \vec{\pi}}
  {\eta \sum_{i=1}^{n} \sum_{k=1}^K z_{ik} (y_i - \mu_k)^2  - \sum_{i=1}^n \sum_{k=1}^K z_{ik} \log \pi_k }
  {\label{eq:map-opt}}
  {}
  \addConstraint{\sum_{k=1}^K \pi_k}{=1}
  \addConstraint{\sum_{k=1}^K z_{ik}}{=1,\quad}{i=1,\ldots,n}
  \addConstraint{M_k^L \leq \mu_k}{\leq M_k^U,\quad}{k=1,\ldots,K}
  \addConstraint{\pi_k}{\geq 0,\quad}{k=1,\ldots,K}
  \addConstraint{z_{ik}}{\in \{0,1\},\quad}{i=1, \ldots, n}
  \addConstraint{}{}{k=1, \ldots, K},
\end{mini}
\normalsize

where \(\eta = \frac{1}{2 \sigma^2}\) is the precision, and $M_k^L$ and $M_k^U$ are real numbers.
In a fully Bayesian setting, even if the \gls{map} clustering is of less interest than the full distribution function, the \gls{map} clustering can still be useful as an initial value for a posterior sampling algorithm as suggested by \citet{gelman1996markov}.

\paragraph{Biconvex mixed-integer nonlinear programming}
While maximum a posterior clustering for a Gaussian mixture model is a well-known problem, it does not fall neatly into any optimization problem formulation except the broadest class: mixed-integer nonlinear programming.
Here we show that, in fact, maximum a posteriori clustering for the Gaussian mixture model is in a more restricted class of optimization problems---biconvex mixed-integer nonlinear problems.
This fact can be exploited to develop improved inference algorithms.

\begin{theorem}
\label{thm:biconvex-minlp}
Maximum a posteriori clustering under the Gaussian mixture model \eqref{eqn:gen-model} with known covariance is a biconvex mixed-integer nonlinear programming optimization problem.
\end{theorem}
\begin{proof}
We briefly sketch a proof that is provided in full in the Supplementary Material.
A biconvex \gls{minlp} is an optimization problem such that if the integer variables are relaxed, the resulting optimization problem is a biconvex nonlinear programming problem.
Maximum a posteriori clustering under Model~\eqref{eqn:gen-model} can be written as Problem~\eqref{eq:map-opt}.
If $z_{ik} \in \{0,1\}$ is relaxed to $z_{ik} \in [0,1]$, Problem~\eqref{eq:map-opt} is convex in $\{ \bm{\pi}, \bm{\mu} \}$ for fixed $\bm{z}$ and convex  in $\bm{z}$ for fixed $\{\bm{\pi}, \bm{\mu}\}$, and it satisfies the criteria of a biconvex nonlinear programming optimization problem~\citep{floudas2000deterministic}.
Because the relaxed problem is a biconvex nonlinear program, the original maximum a posteriori clustering problem is a biconvex MINLP.
\end{proof}

\paragraph{Goal of this work}
Our goal is to solve for the global \gls{map} clustering via a \gls{minlp} over the true posterior distribution domain while only imposing constraints that are informed by our prior knowledge and making controllable approximations.
Importantly, there are no constraints linking $\vec{\pi}$, $\vec{\mu}$, and
$\vec{z}$ such as $\pi_{k} = \frac{1}{n} \sum_{i=1}^n z_{ik}, \ k=1,\ldots,K$ which would be a particular estimator.

\paragraph{Computational complexity}
Problems in the \gls{minlp} class are NP-hard in general, and the \gls{map} problem in particular presents two primary computational challenges~\citep{murty1987some}.
First, as the number of data points increases, the size of the configuration space of $\vec{z}$ increases in a combinatorial manner \citep{nemhauser1988integer}.
Second, the nonlinear objective function can have many local minima.
Despite these worst-case complexity results, \gls{minlp} problems are increasingly often solved to global optimality in practice.
Good empirical performance is often due to exploiting problem-specific structures and powerful algorithmic improvements such as branch-and-bound, branch-and-cut, and branch-and-reduce algorithms \citep{bodic2015how,tawarmalani2005polyhedral}.


\paragraph{Contributions of this work}
This work has three main contributions:
\begin{itemize}
  \item We provide an exact formulation to find globally optimal solutions to the \gls{map} clustering problem and demonstrate that it is tractable for instances formed from small data sets, as well as for larger instances when incorporating side constraints.
  \item We reformulate the original \gls{minlp} by using a piecewise approximation to the objective entropy term and a constraint-based formulation of the mixed-integer-quadratic objective term, thus converting the problem to a \gls{miqp}, and show that it achieves approximately two orders of magnitude improvement in computation time over the \gls{minlp} for larger data set size instances.
  \item We apply our approach to a real breast cancer data set and show that our \gls{miqp} method results in intrinsic subtype assignments that are more biologically consistent than the standard method for subtype assignment.
\end{itemize}
Section~\ref{sec:related_work} describes related work and summarizes the relationship between some existing clustering procedures and the \gls{minlp} formulation.
Section~\ref{sec:bnb} describes our \gls{miqp} formulation and a branch-and-bound algorithm.
Section~\ref{sec:experiments} reports the results of comparisons between our methods and standard clustering procedures.
Section~\ref{sec:discussion} presents a discussion of the results and future work.

\section{Related Work}
\label{sec:related_work}

Many \gls{map} clustering methods can be interpreted as a specific combination of a relaxation of Problem~\eqref{eq:map-opt} and a search algorithm for finding a local or global minimum. \Cref{tbl:related-work-summary} summarizes these relationships.
\begin{table*}[htbp]
    \centering
  \caption{Summary of approximation methods for \gls{map} clustering in an optimization framework.}
  \begin{tabular}{@{}rP{5em}P{8em}p{15em}@{}}
    \toprule
    \textbf{Method} & \textbf{Domain \newline Relaxation} & \textbf{Objective \newline Approximation} & \textbf{Search Algorithm} \\
    \midrule
    EM Algorithm & $\checkmark$ & -- & coordinate/stochastic descent \\
    Variational EM & $\checkmark$ & $\checkmark$ & coordinate/stochastic descent \\
    SLSQP & $\checkmark$ & $\checkmark$ & coordinate descent \\
    Simulated Annealing & $\checkmark$ & -- & stochastic descent \\
    Learning Parameters via MIO & $\checkmark$ & $\checkmark$ & coordinate descent\\
    \bottomrule
  \end{tabular}
  \label{tbl:related-work-summary}
  \vspace{-1em}%
\end{table*}

\paragraph{EM algorithm}
The EM algorithm relaxes the domain such that $z_{ik} \in [0,1]$ instead of $z_{ik} \in \{0,1\}$.
The decision variables of the resulting biconvex optimization problem are partitioned into two groups: $\{\vec{z}\}$ and $\{\vec{\mu}, \vec{\pi}\}$.
The search algorithm performs coordinate ascent on these two groups.
There are no guarantees for the global optimality of the estimate produced by the EM algorithm.
While \citet{balakrishnan2017statistical} showed that the global optima of a mixture of well-separated Gaussians may have a relatively large region of attraction, \citet{chi2016local} showed that inferior local optima can be arbitrarily worse than the global optimum.\footnote{Figure 1 of \citet{chi2016local} illustrates the complexity of the likelihood surface for the \gls{gmm}.}

\paragraph{Variational EM}
The variational EM algorithm introduces a surrogate function $q(\vec{z}, \phi | \vec{\xi})$ for the posterior distribution $p(\vec{z}, \vec{\phi} | \vec{y})$ \citep{beal2003variational}.
First, the surrogate is fit to the posterior by solving $\hat{\vec{\xi}} \in \arg \min_{\vec{\xi}}\ \textrm{KL}(q(\phi, \vec{z} | \xi)\ ||\ p(\phi, \vec{z} | \vec{y}))$.
Then the surrogate is used in place of the posterior distribution in the original optimization problem $\hat{\phi}, \hat{\vec{z}} \in \arg \min_{\phi, \vec{z}}\ \log q(\theta, \vec{z} | \vec{\xi})$.
The search algorithm performs coordinate ascent on $\{\phi, \vec{z}\}$ and $\vec{\xi}$.
The computational complexity is improved over the original \gls{map} problem by selecting a surrogate that has favorable structure (linear or convex) and by relaxing the domain of the optimization problem.
This surrogate function approach has existed in many fields; it is alternatively known as majorization-minimization \citep{lange2000optimizationtransfer} and has deep connections with Franke-Wolfe gradient methods and block coordinate descent methods \citep{mairal2013optimization}.
The domain of the problem can be viewed as a marginal polytope and outer approximations of the marginal polytope lead to efficient sequential approximation methods that have satisfying theoretical properties \citep{wainwright2007graphical}.

\paragraph{SLSQP}
\gls{slsqp} is a popular general-purpose constrained nonlinear optimization method that uses a quadratic surrogate function to approximate the Lagrangian~\citep{nocedal2006numerical}.
In \gls{slsqp}, the surrogate function is a quadratic approximation of the Lagrangian of the original problem.
The domain of the original problem is also relaxed so that the constraint cuts it generates are approximated by linear functions.
Like variational EM, \gls{slsqp} iterates between fitting the surrogate function and optimizing over the decision variables.
Quadratic surrogate functions have also been investigated in the context of variational EM for nonconjugate models \citep{braun2010variational, wang2013variational}.

\paragraph{Simulated annealing}
Simulated annealing methods are theoretically guaranteed to converge to a global optimum of a nonlinear objective.
However, choosing the annealing schedule for a particular problem is challenging and the guarantee of global optimality only exists in the limit of the number of steps; there is no general way to choose the annealing schedule or monitor convergence~\citep{andrieu2003introduction}.
Furthermore, designing a sampler for the binary $\vec{z}$ can be challenging unless the domain is relaxed.
Even so, modern simulated annealing-type methods such a basin hopping have shown promise in practical applications~\citep{wales1997global}.

\paragraph{Branch-and-bound}
In many practical optimization problems featuring combinatorial structure, it is critical to obtain the global optimum with a certificate of optimality, that is, with a (computational) proof that no better solution exists.
For these situations, branch-and-bound methods, first proposed by \citet{land1960automatic}, have seen the most success.
While \gls{minlp} problems remain NP-hard in general, the scale of problems that can be solved to global optimality has increased dramatically in the past 20 years~\citep{bertsimas2017logistic}.
The current state-of-the-art solver for general \glspl{minlp} is the \gls{baron}.
\gls{baron} exploits general problem structure to branch-and-bound with domain reduction on both discrete and continuous decision variables \citep{sahinidis2017baron}.

\paragraph{Learning parameters via \gls{mio}}
\citet{bandi2019learning} recently described a mixed-integer optimization formulation of the parameter estimation problem for the Gaussian mixture model. 
Conditional on the parameter estimates, they computed the one-sample-at-a-time \gls{map} assignments for out-of-sample data. 
They convincingly demonstrate that a mixed-integer optimization approach can outperform the EM algorithm in terms of out-of-sample accuracy for real-world data sets.

\citet{bandi2019learning} and our paper solve related, yet different problems. Their primary objective is density estimation---to find the optimal parameters of the Gaussian mixture model. Our primary objective is MAP clustering---to find an optimal maximum a posteriori assignment of data points to clusters and associated distribution parameters. To assign data points to clusters, \citet{bandi2019learning} first estimates the model parameters and then computes the maximum a posteriori assignments conditioned on the estimated model parameters. Even so, \citet{bandi2019learning} showed that an optimization framework can yield real benefits for learning problems.

\vspace{-1em}
\section{Branch-and-Bound for MAP Clustering}
\label{sec:bnb}

Branch-and-bound solves the original problem with integer variables by solving a sequence of relaxations that partition the search space.
The branch-and-bound algorithm provides for the opportunity to exclude large portions of the search space at each iteration, and if the branching strategy is well-suited to the problem it can substantially reduce the actual computation time for real problems~\citep{grotschel2012optimization}.

Our innovations to the standard \gls{minlp} fall into three categories: changes to the domain, changes to the objective function, and changes to the branch-and-bound algorithm.
For the domain constraints, we formulate a symmetry-breaking constraint, specific estimators for $\vec{\pi}$ and $\vec{\mu}$, tightened parameter bounds, and logical constraints.
For the objective function, we formulate the prior distribution $p(\vec{\pi}, \vec{\mu})$ as a regularizer, a piecewise linear approximation to the nonlinear logarithm function, and an exact mapping of a mixed-integer quadratic term as a set of constraints.

A common difficulty in obtaining a global optimum for Problem~\eqref{eq:map-opt} is that the optimal value of the objective is invariant to permutations of the component ordering.
In the \gls{minlp} framework, we eliminate such permutations from the feasible domain by adding a simple linear constraint $\pi_1 \leq \pi_2 \leq \cdots \leq \pi_K.$
This set of constraints enforces an ordering, thereby reducing the search base and improving computational performance.

Though Problem~\ref{eq:map-opt} does not specify any particular form for the estimators of $\vec{\pi}$ or $\vec{\mu}$, it may be of interest to specify the estimators with equality constraints.
For example the standard estimators of the EM algorithm are:
\begin{equation*}
        \pi_k = \frac{1}{n} \sum_{i=1}^n z_{ik},\quad \text{and}\quad \vec{\mu}_k = \frac{\sum_{i=1}^n \vec{y}_i z_{ik}}{\sum_{i=1}^n z_{ik}} \; \forall k.
\end{equation*}
The resulting optimization problem can be written entirely in terms of integer variables and the goal is to find the optimal configuration of $\vec{z}$.
Note that Problem~\eqref{eq:map-opt} does not specify these constraints and we do not enforce these particular estimators in our experiments in Section~\ref{sec:experiments}.

A conservative bound on $\vec{\mu}_k$ is $[\min(\vec{y}), \max(\vec{y})]$ and this bound is made explicit in Problem~\eqref{eq:map-opt}.
Of course, prior information may be available that informs and reduces the range of these bounds.
Placing more informative box constraints on the parameter values has been shown by \citet{bertsimas2016best} to greatly improve the computational time of similar \glspl{minlp}.

An important aspect of using the \gls{minlp} formulation as a Bayesian procedure is the ability to formulate logical constraints that encode rich prior information.
These constraints shape the prior distribution domain in a way that is often difficult to express with standard probability distributions.
Problem~\eqref{eq:map-opt} already has one either/or logical constraint specifying that a data point $i$ must belong to exactly one component $k$.
One can specify that data point $i$ and $j$ must be assigned to the same component,
$z_{ik} = z_{jk},\ \forall k$
or that they must be assigned to different components, $z_{ik} + z_{jk} \leq 1,\ \forall k$.
A non-symmetric constraint can specify that if data point $j$ is assigned to component $k$, then $i$ must be assigned to $k$, $z_{jk} \leq z_{ik}$;
on the other hand, if $i$ is not assigned to component $k$, then $j$ cannot be assigned to component $k$.
A useful constraint in the context of the \gls{gmm} is the requirement that each component has a minimum number of data points assigned to it, $\sum_{i=1}^n z_{ik} \geq L,\quad \text{for}\ k=1,\ldots,K$,
where $L$ is a specified number of data points.
Additional logical constraints can be formulated in a natural way in the \gls{minlp} such as: packing constraints (from a set of data points $\mathcal{I}$, select at most $L$ to be assigned to component $k$), $\sum_{i \in \mathcal{I}} z_{ik} \leq L$; partitioning constraints (from a set of data points, select exactly $L$ to be assigned to component $k$) $\sum_{i \in \mathcal{I}} z_{ik} = L$; and covering constraints (from a set of data points select at least $L$) $\sum_{i \in \mathcal{I}} z_{ik} \geq L$.

\subsection{Approximation via \gls{miqp}}
Recall the objective function of Problem~\eqref{eq:map-opt} is:
\small
\begin{equation*}
    f(\vec{z}, \vec{\pi}, \vec{\mu}; \vec{y}, \eta) = \eta \sum_{k=1}^K \sum_{i=1}^{n} z_{ik} (\vec{y}_i - \mu_k)^2  - \sum_{i=1}^n \sum_{k=1}^K z_{ik} \log \pi_k.
\end{equation*}
\normalsize
The objective function is mixed-integer and nonlinear due to the product of $z_{ik}$ and $\mu_k^2$ in the first term and the product of binary $z_{ik}$ and the $\log$ function in the second term.

The objective function can be shaped by the prior distribution, $p(\vec{\mu}, \vec{\pi})$.
In the formulation of Problem~\eqref{eq:map-opt} a uniform prior was selected such that $p(\vec{\mu}, \vec{\pi}) \propto 1$ and the prior does not affect the optimizer.
But, an informative prior such as a multivariate Gaussian $\vec{\mu} \sim \textrm{MVN}(\vec{0}, \vec{S})$ could be used to regularize $\vec{\mu}$, or a non-informative prior such as Jeffrey's prior could be used in an objective Bayesian framework.

The template-matching term in the objective function has two nonlinearities: $2 y_i z_{ik} \mu_k$ and $z_{ik} \mu_{k}^2$.
Such polynomial-integer terms are, in fact, commonly encountered in problems such as capital budgeting, scheduling problems, and many others~\citep{glover1975improved}.
We have from Problem~\eqref{eq:map-opt} that $M_k^L \leq \mu_k \leq M_k^U$ when the component bounds are not the same.
Given $z_{ik}$ is a binary variable, we can rewrite the term $\sum_k z_{ik}(\vec{y}_i - \vec{\mu}_k)^2$ as $(\vec{y}_i - \sum_k z_{ik}\vec{\mu}_k)^2$ because $\sum_k z_{ik}y_i = y_i$ and each data point is constrainted to be assigned to exactly one component.
Then, we introduce a new continuous variable $t_{ik} = z_{ik}\mu_k$ which is implicitly enforced with the following four constraints for each $(i,k)$:
\begin{align}
  M_k^L z_{ik} \leq t_{ik} & \leq M_k^U z_{ik}, \\
  \mu_k - M_k^U (1 - z_{ik})  \leq t_{ik} & \leq \mu_k - M_k^L (1 - z_{ik}).
\end{align}
Now, the objective function term $(y_i - \sum_{k}{t_{ik}})^2$ is simply quadratic in the decision variables and the additional constraints are linear in the decision variables.

The cross-entropy term, $z_{ik} \log \pi_k$, is a source of nonlinearity in Problem~\eqref{eq:map-opt}.
Approximating this nonlinearity with a piecewise linear function has two benefits.
First, the accuracy of the approximation can be controlled by the number of breakpoints in the approximation.
This results in a single parametric ``tuning knob'' offering direct control of the tradeoff between computational complexity and statistical accuracy.
Second, sophisticated methods from ordinary and partial differential equation integration or spline fitting can be brought to service in selecting the locations of the breakpoints of the piecewise-linear approximation.
It may be possible to set breakpoint locations adaptively as the optimization iterations progress to gain higher accuracy in the region of the \gls{map} and the approximation can be left coarser elsewhere.
Indeed, good convex optimization methods have been developed to solve for the optimal breakpoint locations~\citep{magnani2008convex, bandi2019learning}, but we employ regular breakpoints in the \gls{miqp} implemented for our experiments in \Cref{sec:experiments} for simplicity.

\section{Experiments}
\label{sec:experiments}
In this section, we report results on standard data sets and a real breast cancer gene expression data set.

\subsection{Comparison to Local Search Methods}
\label{sec:experiments1}
We compare variations of our proposed approach to EM, SLSQP, and simulated annealing on several standard clustering data sets.
Our primary interest in this work is in closing the gap between the upper and lower bounds thus achieving a certificate of global optimality for clustering assignments under the Gaussian mixture model.
We note that our approach typically obtains a strong feasible solution very quickly; it is the proof of global optimality that is obtained from closing the lower bound that consumes the majority of the computational time.
For methods that admit constraints, we employ the following additional constraints: $\pi_1 \leq \cdots \leq \pi_k$ and $\sum_{i=1}^n z_{ik} \geq 1,\ \forall k$.

\paragraph{Data collection and preprocessing}
We obtained the Iris (\texttt{iris}, $n=150$), Wine Quality (\texttt{wine}, $n=178$), and Wisconsin Breast Cancer (\texttt{brca/wdbc}, $n=569$) data sets from the UCI Machine Learning Repository \citep{dua2019uci}.
A 1-d projection of \texttt{iris} was obtained by projecting on the first principal component (designated \texttt{iris1d}), and only the following features were employed for the \texttt{brca} data set: worst area, worst smoothness, and mean texture.
The \texttt{wine} data set is 13-dimensional and the \texttt{brca} data set is 3-dimensional.
Since our goal is to obtain the global \gls{map} clustering given the data set rather than a prediction, all of the data was used for clustering.
The component variance was fixed to $0.4^2$ for the \texttt{iris1d} data and for \texttt{wine} and \texttt{brca} the precision matrices were fixed to the average of the true component precision matrices.

\paragraph{Experimental protocol}
The EM, SLSQP, and simulated annealing algorithms were initialized using the $K$-means algorithm; no initialization was provided to the \gls{minlp} and \gls{miqp} methods.
The EM, SLSQP, SA experiments were run using algorithms defined in \texttt{python/scikitlearn} and all variants of our approach were implemented in GAMS \citep{bisschop1982gams, gams2017manual}, a standard optimization problem modeling language.
The estimates provided by EM, SLSQP and SA are not guaranteed to have $z_{ik} \in \{0,1\}$ so we rounded to the nearest integral value, but we note that solutions from these methods are not guaranteed to be feasible for the \gls{map} problem.
\gls{minlp} problems were solved using BARON \citep{tawarmalani2005polyhedral, sahinidis2017baron} and \gls{miqp} problems were solved using Gurobi \citep{gurobi2018manual}; both are state-of-the-art general-purpose solvers for their respective problems \citep{kronqvist2019review}.
We report the upper bound (best found) and lower bound (best possible) of the negative \gls{map} value if the method provides it.
For the timing results in \Cref{fig:computation-time-vs-sample-size} we ran both methods on a single core, and for the results in \Cref{tbl:comparison_results} we ran all methods on a single core except for the \gls{miqp} method which allowed multithreading, where 16 cores were used.
The computational time for all methods was limited to 12 hours.
The metrics for estimating $\vec{\pi}$, $\vec{\mu}$, and $\vec{z}$ are shown in results Table~\ref{tbl:comparison_results}.

\paragraph{Performance with increasing sample size}
We assessed convergence to the global optimum for the \texttt{iris1d} data set restricted to $n=45$ data points (15 in each of the three true components).
\Cref{fig:upper-lower-bound-convergence} shows that the upper bound converges very quickly to the global optimum, and it takes the majority of the time to (computationally) prove optimality within a predetermined $\epsilon$ threshold.
\Cref{fig:computation-time-vs-sample-size} shows the computational time versus sample size for the \texttt{iris1d} data set for the \gls{minlp} and the \gls{miqp} methods.
Our \gls{miqp} approach reduces computational time by a multiplier that increases with sample size and is around two orders of magnitude for $n=45$.

\begin{figure}[!ht]
  \centering
  \small
  \begin{minipage}{1.0\columnwidth}
  \begin{subfigure}[t]{0.48\textwidth}\centering
    \includegraphics[width=\textwidth]{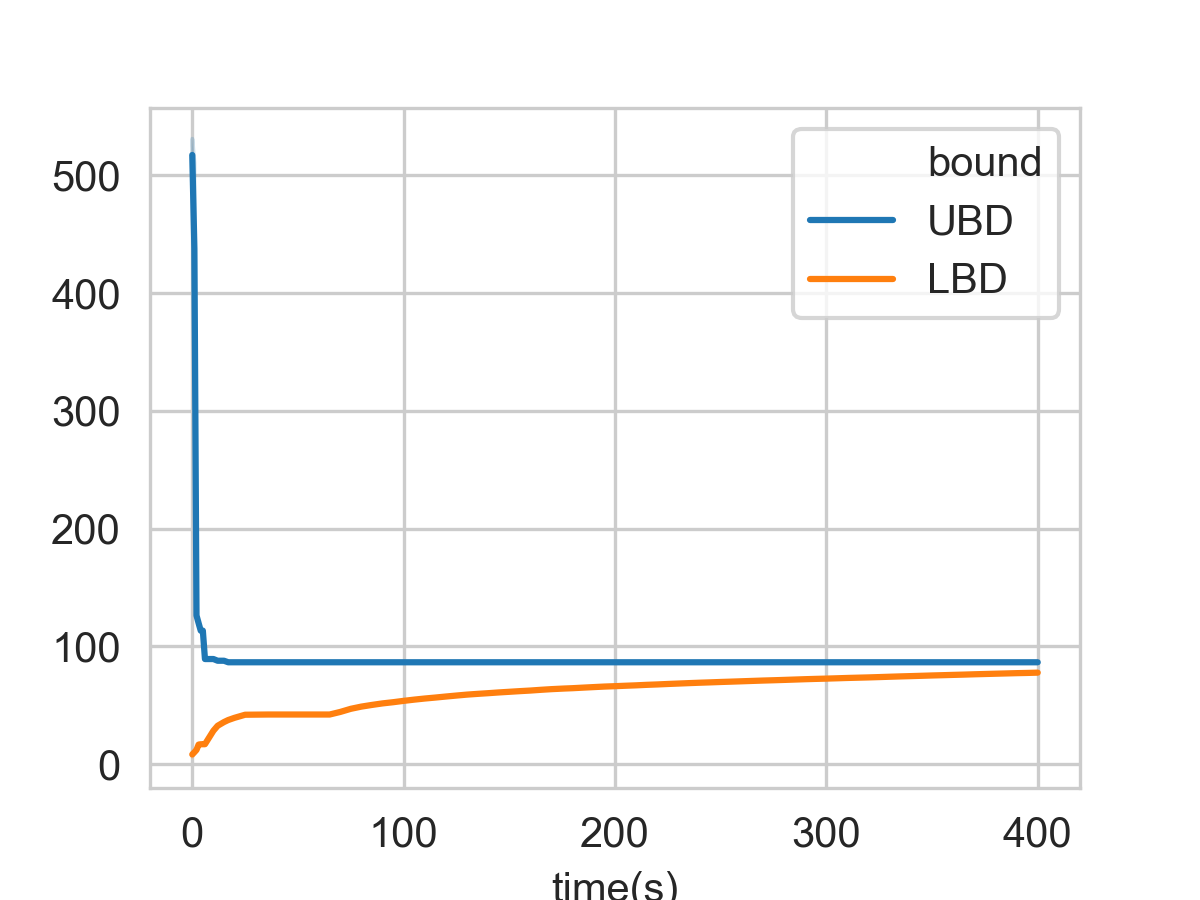}
    \caption{Convergence of upper and lower bounds for \gls{minlp} method shows that optimal solution is found quickly, with the majority of time is spent proving global optimality of the \gls{map} estimate.}
    \label{fig:upper-lower-bound-convergence}
  \end{subfigure}
  \;
  \begin{subfigure}[t]{0.48\textwidth}\centering
    \includegraphics[width=\textwidth]{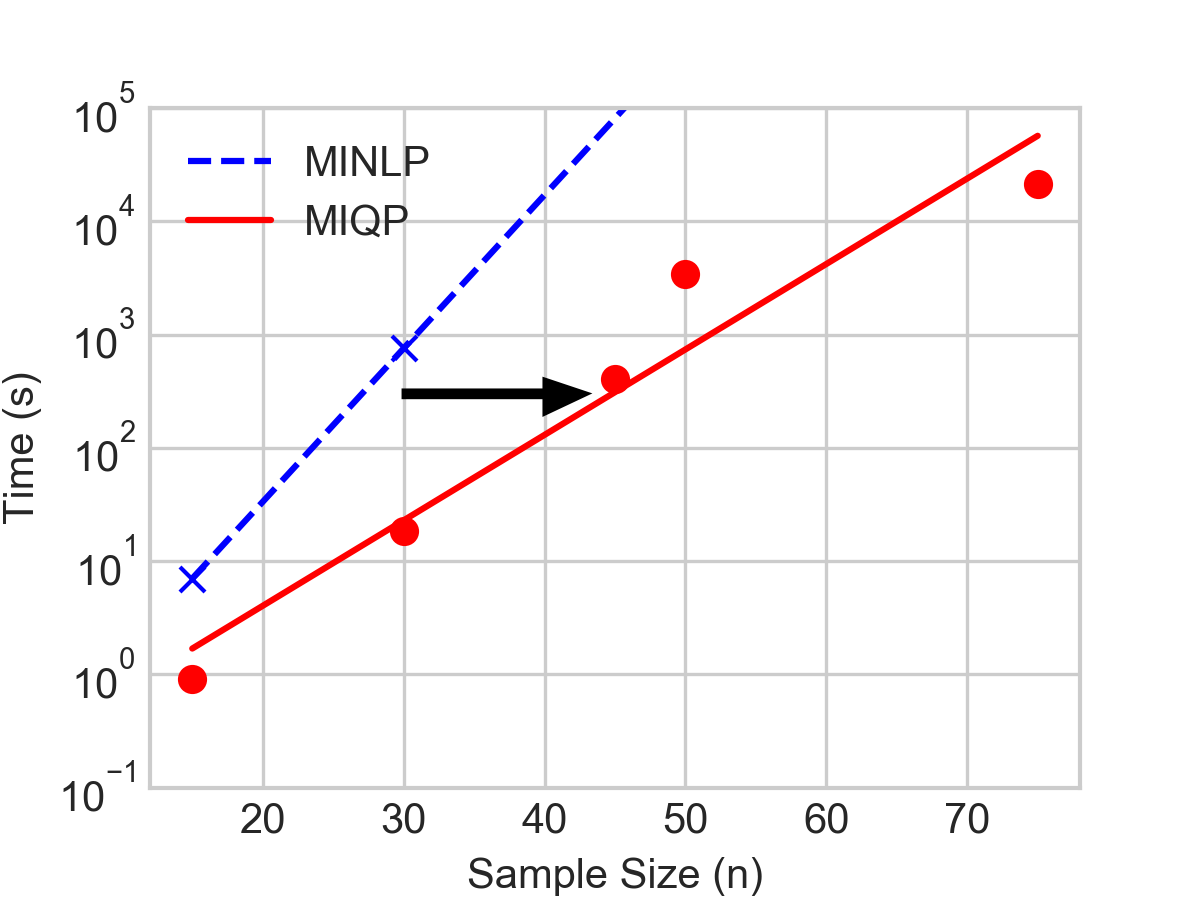}
    \caption{Our \gls{miqp} approach improves upon the \gls{minlp} solution computation time as shown by the shift of the computation (log) time versus sample size curve to the right.}
    \label{fig:computation-time-vs-sample-size}
  \end{subfigure}
  \caption{Global Convergence of \gls{minlp} and Time Comparison between \gls{minlp} and \gls{miqp}}
  \end{minipage}
  \vspace{-1em}
\end{figure}

\Cref{tbl:comparison_results} shows the comparison of our proposed branch-and-bound methods (MINLP, MIQP) with standard local search methods (EM, SLSQP, SA).
Our primary interest lies in achieving a measure of convergence to the global optimum and the relative gap indicates the proximity of the upper and lower bounds.
On all of the data sets, all methods find roughly the same optimal value.
The \gls{minlp} method consistently has a fairly large gap and the \gls{miqp} method has a much smaller gap indicating that it provides a tighter bound on the globally optimal value.
\begin{table*}[htbp]
  \centering
  \caption{Comparison of expectation-maximization (EM), sequential least squares programming (SLSQP),  basin-hopping simulated annealing (SA), branch-and-bound on the \gls{minlp} using BARON (MINLP), branch-and-bound on the \gls{miqp} using Gurobi (MIQP). The solution reported for \gls{minlp} and \gls{miqp} are the best feasible solution found within time limit.  The total variation distance metric is used for $\vec{\pi}$, the $L_2$ distance is used for $\vec{\mu}$, and the average (across samples) total variation distance is used for $\vec{z}$.}
  \begin{tabular}{@{}rrcccccc@{}}
    \toprule
      & & \multicolumn{3}{c}{Local} & \phantom{a} & \multicolumn{2}{c}{Global (BnB)} \\
    \cmidrule{3-5} \cmidrule{7-8}
    Data Set & Metric & EM & SLSQP & SA &  & MINLP & MIQP \\
    \midrule
    iris (1 dim) & $-\log$ MAP & 280.60 & 287.44 & 283.28 & & 280.02 & 282.71 \\
    & LBD & --- & --- & ---&  & 9.27 & 161.60 \\
    & $\sup |\hat{\vec{\pi}} - \vec{\pi}|$ & 0.075 & 0.013 & 0.000 & & 0.093 & 0.165 \\
    & $\|\hat{\vec{\mu}} - \vec{\mu}\|_2$ & 0.278 & 0.065 & 0.277 & & 0.356 & 0.356 \\
    & $\nicefrac{1}{n} \sum_i \sup |\hat{\vec{z}_i} - \vec{z}_i|$ & 0.067 & 0.067 & 0.087 & & 0.093 & 0.093 \\
    \\
    wine (13 dim) & $-\log$ MAP & 1367.00 & 1368.71 & 1368.71 & & 1366.85 & 1390.13 \\
    & LBD & --- & --- & --- & & \num{-2.2e5} & 183.42  \\
    & $\sup |\hat{\vec{\pi}} - \vec{\pi}|$ & 0.005 & 0.066 & 0.066 & & 0.006 & 0.167 \\
    & $\|\hat{\vec{\mu}} - \vec{\mu}\|_2$ & 2.348 & 1.602 & 1.652 & & 1.618 & 14.071 \\
    & $\nicefrac{1}{n} \sum_i \sup |\hat{\vec{z}_i} - \vec{z}_i|$ & 0.006 & 0.006 & 0.006 & & 0.006 & 0.022 \\
    \\
    brca (3 dim) & $-\log$ MAP & 1566.49 & 1662.97 & 1662.97 & & 1566.40 & 1578.49 \\
    & LBD & --- & --- & --- & & \num{-2.7e4} & 332.30 \\
    & $\sup |\hat{\vec{\pi}} - \vec{\pi}|$ & 0.167 & 0.127 & 0.127 & & 0.169 & 0.122  \\
    & $\|\hat{\vec{\mu}} - \vec{\mu}\|_2$ & 394.07 & 321.11 & 320.60 & & 401.47 & 418.05 \\
    & $\nicefrac{1}{n} \sum_i \sup |\hat{\vec{z}_i} - \vec{z}_i|$ & 0.169 & 0.139 & 0.139 & & 0.169 & 0.174  \\
    \bottomrule
  \end{tabular}
  \label{tbl:comparison_results}
  \vspace{-1em}%
\end{table*}

\paragraph{EM algorithm with multiple restarts}
We explored whether the EM algorithm can identify the global optimum if restarted multiple times from random initial values.
We restarted the EM algorithm and ran to convergence for 12 hours and retained the best local \gls{map} value.
For the \texttt{iris1d} data set with $n=45$, the best local \gls{map} identified by the EM algorithm was 194.9582 and the global \gls{map} identified by \gls{miqp} was 194.9455.
The global \gls{map} clustering was identified by the \gls{miqp} algorithm in 17 seconds, an additional 125.19 seconds were spent proving global optimality (converging the lower bound as shown in \Cref{fig:upper-lower-bound-convergence}).
This shows that EM may get trapped in attractive, albeit local, minima, and so is unable to prove global optimality.

\paragraph{Comparison to Bandi et. al.}
We implemented the \citet{bandi2019learning} method with the Kolmogorov-Smirnov distance as the discrepancy measurement and set the stopping criterion to $\epsilon =0.01$. 
The in-sample cluster assignment accuracy on the \texttt{iris1d} and \texttt{brca} data sets are 90.0\% and 74.4\% using \citet{bandi2019learning}. 
For the \texttt{wine} data set, the clustering accuracy is 58.4\%. 
Our implementation matches the results reported in \citet{bandi2019learning} for \texttt{iris} and \texttt{brca}; they did not report results on \texttt{wine}.
In comparison, Table~\ref{tbl:comparison_results} shows that our in-sample assignment accuracy is 90.7\%, 83.1\% and 99.4\% for \texttt{iris1d}, \texttt{brca}, and \texttt{wine}.
To test whether the discrepancy on the \texttt{wine} data set is due to the estimation of the variance parameters, we fixed the variance parameters and cluster proportion parameters to their true values for the \citet{bandi2019learning} method.
The cluster assignment accuracy improved but still does not achieve the same level of accuracy as our \gls{miqp} method.
The remaining possible causes for the discrepancy are that the \citet{bandi2019learning} method uses the Kolmogorov-Smirnov distance or total variation distance to estimate parameters rather than maximum likelihood or maximum a posteriori cluster assignments.
To underscore the fundamental difference between our approaches, our approach is focused on the in-sample cluster assignment---an important problem in Bayesian statistics and biological data analysis, whereas the \citet{bandi2019learning} approach is focused on the out-of-sample prediction problem---a different, but no less important problem.

\subsection{Clustering Breast Cancer Gene Expression Data}
\label{sec:experiments2}

We evaluated our proposed approach on Prediction Analysis of Microarray 50 (\texttt{pam50}) gene expression data set.
The PAM50 gene set is commonly used to identify the ``intrinsic'' subtypes of breast cancer among luminal A (LumA), luminal B (LumB), HER2-enriched (Her2), basal-like (Basal), and normal-like (Normal).
Different subtypes lead to different treatment decisions, so it is critical to identify the correct subtype.
The PAM algorithm finds centroids under a Spearman's rank correlation distance~\citep{pamalg2002, pam2009}.
Our primary interest with the experiment is to show the effects of incorporating constraints into the clustering problem.
In this data set, we have important side information that we incorporate and find that it significantly improves computational time and accuracy.

\paragraph{Data collection and preprocessing}
We used the \texttt{pam50} data set ($n=232$, $d=50$) obtained from UNC MicroArray Database \citep{UNCMicro}.
The \texttt{pam50} dataset contains 139 subjects whose intrinsic subtypes are known, and 93 subjects whose intrinsic subtypes are unknown.
Missing data values (2.65\%) were filled in with zeros.
A 4-d projection of \texttt{pam50} was obtained by projecting on the first 4 principal components (designated \texttt{pam50-4d}), preserving 58.6\% of the information variation.
The precision matrix was fixed to the average of the estimates intrinsic subtype component precision matrices.

\paragraph{Experimental Protocol}
We ran the PAM algorithm using R code provided by the authors as supplementary material.
\gls{minlp} and \gls{miqp} methods employed additional constraints to Problem~\eqref{eq:map-opt} such that 139 samples were assigned to known subtypes.
All variants of our approaches were implemented in the same environment described in \Cref{sec:experiments1}, except \gls{miqp} was run on 32 cores.

\paragraph{Hard constraints improve computational time}
We assessed convergence to the global optimum for the \texttt{pam50-4d} data set using \gls{miqp} with and without the constraints on 139 samples with known intrinsic subtypes.
Without constraints, the optimality gap was 99.83\%, and the addition of constraints allowed the method to achieve a 32.35\% optimality gap for the same computation time.
The incorporation of hard constraints greatly improves the computational efficiency.

\paragraph{Comparison to PAM algorithm}
\Cref{fig:pamAssignment} shows a comparison of cluster assignments of our methods (\gls{minlp}, \gls{miqp}) with the PAM algorithm.
For 139 samples with known instrinsic subtypes, assignments from \gls{minlp} and \gls{miqp} methods have 100\% accuracy, while PAM  accuracy is 94\%.
This is due to the capability of \gls{minlp} and \gls{miqp} to incorporate prior constraints in the clustering problem.
For the 93 samples with unknown subtypes, \gls{minlp} assignments have 68\% concordance with the PAM algorithm, and \gls{minlp} has 89\% concordance with \gls{miqp} assignments.
\begin{figure}[htbp]
  \centering
  \includegraphics[width=\columnwidth]{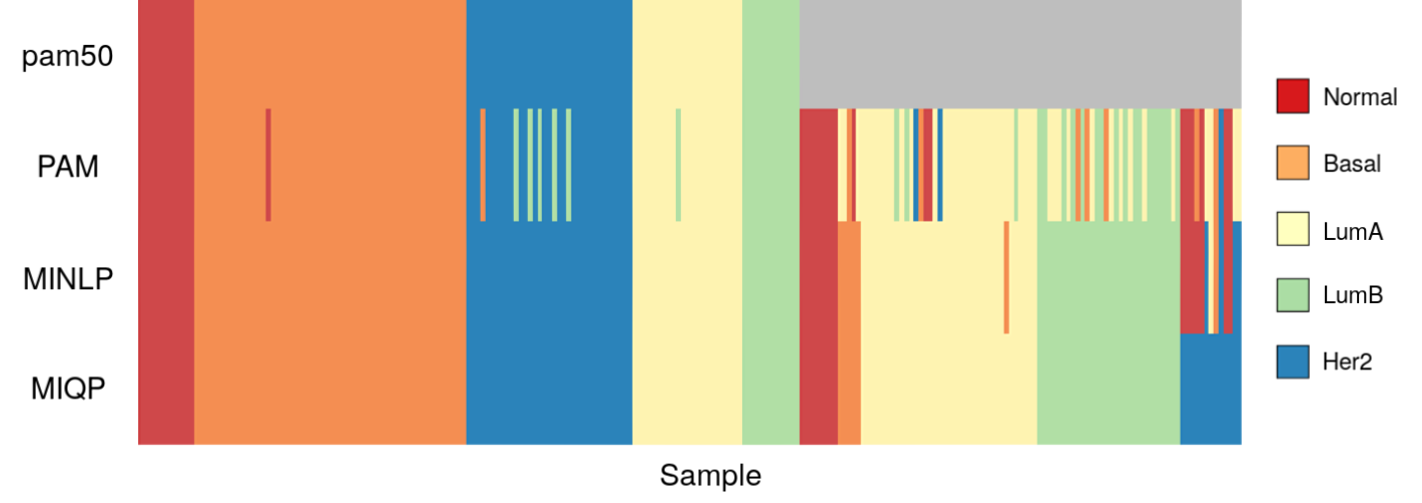}
  \caption{Breast cancer intrinsic subtypes assignments. \texttt{pam50} shows true assignments for 139 samples and unknown assignments (gray) for 93 samples. PAM is the centroid algorithm in \citet{pamalg2002}; \gls{minlp} and \gls{miqp} are our methods.}
  \label{fig:pamAssignment}
 \vspace{-1em}
\end{figure}

\paragraph{Biological interpretation of \gls{miqp} and PAM cluster assignments}
We selected two samples that were assigned to different clusters by the PAM algorithm and our \gls{miqp} method to explore whether the cluster assignments made by \gls{miqp} are more or less biologically meaningful (see Supplementary Material for plots of other samples).
\Cref{fig:expression_comparison} shows the gene expression profile for sample 13 which has a known intrinsic subtype (\Cref{fig:known}) and sample 24 which has an unknown intrinsic subtype (\Cref{fig:unknown1}).

\begin{figure}[htbp]
  \centering
  \begin{minipage}{1.0\columnwidth}
  \begin{subfigure}[t]{1.0\textwidth}\centering
    \includegraphics[width=\textwidth]{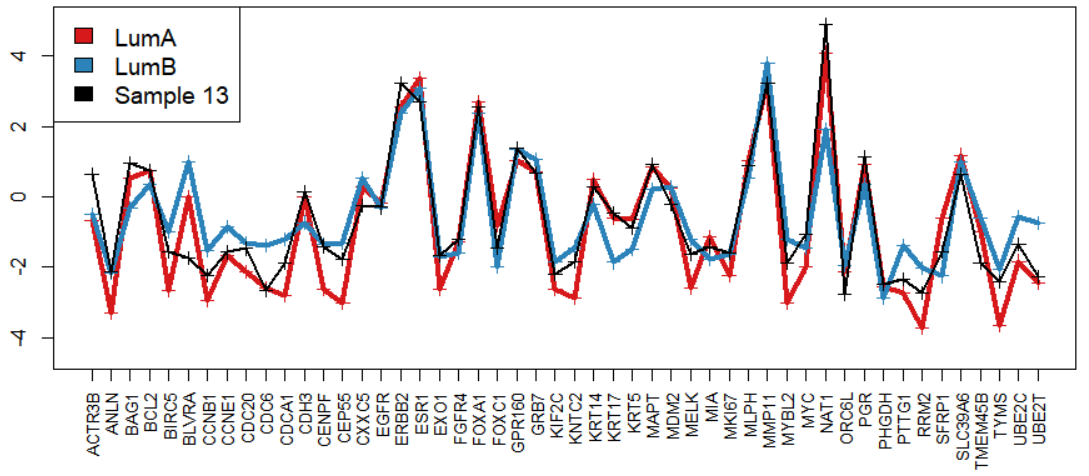}
    \caption{Gene expression data comparison of sample 13, average of LumA and average of LumB. The true subtype is LumA, MIQP correctly assigned it to LumA, and PAM incorrectly assigned it to LumB.}
    \label{fig:known}
  \end{subfigure}
  \par\bigskip
  \begin{subfigure}[t]{1.0\textwidth}\centering
    \includegraphics[width=\textwidth]{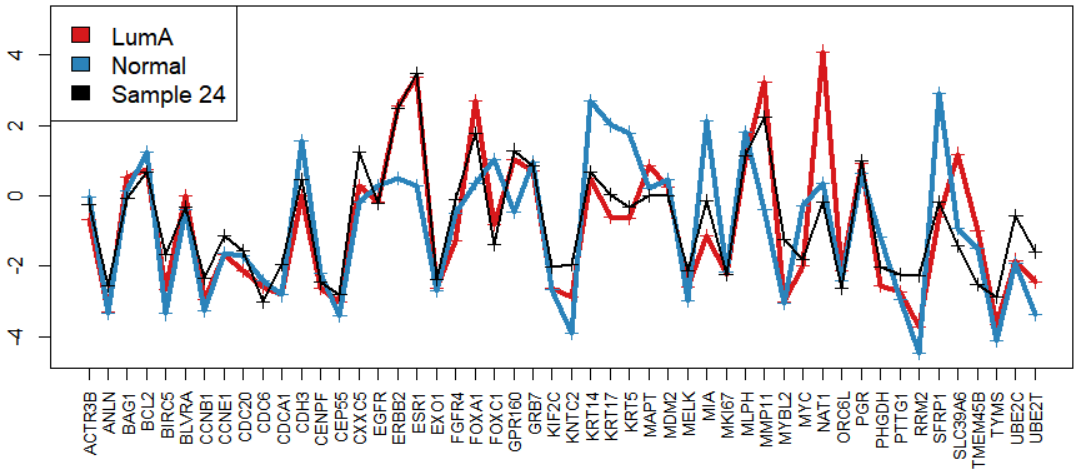}
    \caption{Gene expression data comparison of sample 24, average of LumA, and average of Normal. The true subtype is unidentified, MIQP assigned it to LumA, and PAM assigned it to Normal.}
    \label{fig:unknown1}
  \end{subfigure}
    \caption{Comparison of gene expression profiles with average intrinsic subtype profiles for samples 13 and 24.}
    \label{fig:expression_comparison}
  \end{minipage}
  \vspace{-2em}
\end{figure}

Sample 13 is known to be a Luminal A (LumA) subtype and is correcty assigned to LumA by the \gls{miqp} method, but incorrectly assigned to LumB by PAM. The expression profile clearly matches the LumA average more closely for CCNB1, CCNE1, CDC6 and NAT1. NAT1 has been shown to be elevated in individuals with estrogen receptor positive (ER+) breast cancers~\citep{wang2019loss} and 30-45\% of Luminal A cancers are ER+~\citep{voduc2010breast}. Therefore, the assignment to Luminal A is more biologically consistent.

Sample 24 is an unknown intrinsic subtype in the \texttt{pam50} data set; it is assigned to the Normal subtype by the PAM algorithm and to the LumA subtype by the \gls{miqp} method.
Clearly, the expression profile of the unknown sample shows an overexpression of ERBB2, ESR1, and the KRT cluster (KRT14, KRT17, and KRT5) which are consistent with the average LumA expression profile and inconsistent with the average Normal expression profile.
ERBB2 is commonly known as HER2 and overexpression of HER2 would indicate the sample might be assigned to the HER2 intrinsic subtype.
The moderate keratin (KRT) expression levels (KRT14, KRT17, KRT5) are clearly more consistent with LumA (overexpression has been associated with the basal subtype~\citep{fredlund2012gene, gusterson2005basal}).
Even though the expression of the NAT1 gene in the unknown sample is different than the LumA average, the overall pattern of expression is  consistent with LumA.
All of this evidence suggests that this sample may be heterogeneous with multiple oncogenic driving events (including HER2 overexpression); and it is almost certainly not Normal.

\section{Discussion}
\label{sec:discussion}

We formulated the \gls{map} clustering problem as an \gls{minlp} and cast several local \gls{map} clustering methods in the framework of the \gls{minlp}.
We identified three aspects of the \gls{minlp} that can be adjusted to incorporate prior information and improve computational efficiency.
We suggested an approximation that transforms the \gls{minlp} to an \gls{miqp} and offers control over the approximation error at the expense of computational time.
We applied our approach to a real breast cancer clustering data set and found that incorporating hard constraints in our method significantly both improves computational time and biological interpretation of the resulting sample assignments.

\section*{Acknowledgments}
This work is supported by NSF HDR TRIPODS award 1934846.

\bibliography{library}
\bibliographystyle{apalike}

\appendix

\onecolumn

\section{Biconvex MINLP Proof}
\label{sec:biconvex-full}

In this section we define the \textit{biconvex mixed-integer nonlinear programming} problem and show that Problem~\eqref{eq:map-opt} is such a problem.
In Section~\ref{sec:minlp}, we review the definitions related to mixed-integer nonlinear optimization.
Then, in Section~\ref{sec:biconvex} we define the biconvex optimization problem.
Finally, in Section~\ref{sec:biminlp} we extend definitions from Section~\ref{sec:minlp} and Section~\ref{sec:biconvex} and apply them to Problem~\eqref{eq:map-opt}.

\subsection{Mixed-Integer Nonlinear Optimization}
\label{sec:minlp}

\begin{definition}[MINLP]
A mixed-integer nonlinear programming problem is of the form
\begin{mini}
  {\bm{x}, \bm{y}}{f(\bm{x}, \bm{y})}{}{\label{eqn:minlp}}
  \addConstraint{\bm{g}(\bm{x}, \bm{y})}{\leq 0}
  \addConstraint{\bm{h}(\bm{x}, \bm{y})}{= 0}
\end{mini}
where $x \in \mathcal{X} \subseteq \mathbb{R}^n$, $y \in \mathcal{Y} \subseteq \mathbb{Z}^m$, $f : \mathbb{R}^n \times \mathbb{R}^m \rightarrow \mathbb{R}$, $\bm{g} : \mathbb{R}^n \times \mathbb{R}^m \rightarrow \mathbb{R}^p$, and $\bm{h} : \mathbb{R}^n \times \mathbb{R}^m \rightarrow \mathbb{R}^q$ \citep{belotti2013mixed}.
\end{definition}

\begin{definition}[Convex MINLP]
Problem~\eqref{eqn:minlp} is a \textit{convex MINLP} if $f$ is separable and linear, $f(\bm{x}, \bm{y})  = \bm{c}_1^T \bm{x} + \bm{c}_2^T \bm{y}$, and the constraints $g_j$ are jointly convex in $\bm{x}$ and the relaxed integer variables $y_i \in [0,1]$ for all $j$ and $i$ \citep{kronqvist2019review}.
\end{definition}

\subsection{Biconvex Optimization}
\label{sec:biconvex}

\begin{definition}[Biconvex Set]
The set $\mathcal{B} \subseteq \mathcal{X} \times \mathcal{Y}$ is called a biconvex set on $\mathcal{X} \times \mathcal{Y}$ or biconvex if $\mathcal{B}_{\bm{x}}$ is convex for every $\bm{x} \in \mathcal{X}$ and $\mathcal{B}_{\bm{y}}$ is convex for every $\bm{y} \in \mathcal{Y}$ \citep{Gorski2007}.
The $x-$ and $y-$ sections of $\mathcal{B}$ are defined as $\mathcal{B}_x = \{\bm{y} \in \mathcal{Y} : (\bm{x}, \bm{y}) \in \mathcal{B} \}$ and $\mathcal{B}_y = \{\bm{x} \in \mathcal{X} : (\bm{x}, \bm{y}) \in \mathcal{B} \}$ respectively.
\end{definition}

\begin{definition}[Biconvex Function]
A function $f : \mathcal{B} \rightarrow \mathbb{R}$ on a biconvex set $\mathcal{B} \subseteq \mathcal{X} \times \mathcal{Y}$ is a biconvex function on $\mathcal{B}$ if $f_{\bm{x}}(\bullet) = f(\bm{x}, \bullet) : \mathcal{B}_{\bm{x}} \rightarrow \mathbb{R}$ is a convex function on $\mathcal{B}_{\bm{x}}$ for every fixed $\bm{x} \in \mathcal{X}$ and $f_{\bm{y}}(\bullet) = f(\bullet, \bm{y}) : \mathcal{B}_{\bm{y}} \rightarrow \mathbb{R}$ is a convex function on $\mathcal{B}_{\bm{y}}$ for every fixed $\bm{y} \in \mathcal{Y}$ \citep{Gorski2007}.
\end{definition}

\begin{definition}[Biconvex Optimization Problem]
The problem $\min f(\bm{x}, \bm{y}) : (\bm{x}, \bm{y}) \in \mathcal{B}$
is a biconvex optimization problem if the feasible set $\mathcal{B}$ is biconvex on $\mathcal{X} \times \mathcal{Y}$ and the objective function $f$ is biconvex on $\mathcal{B}$ \citep{Gorski2007}.
\end{definition}

\subsection{Biconvex Mixed-Integer Nonlinear Optimization}
\label{sec:biminlp}

\begin{definition}[Biconvex MINLP]
The problem 
\begin{mini}
  {\bm{x}, \bm{y}}{f(\bm{x}, \bm{y})}{}{\label{eqn:biconvex-minlp}}
  \addConstraint{\bm{g}(\bm{x}, \bm{y})}{\leq 0}
  \addConstraint{\bm{h}(\bm{x}, \bm{y})}{= 0}
\end{mini}
is a \textit{biconvex MINLP} if the feasible set $\mathcal{B}$ and the objective function $f$ are biconvex in $\bm{x}$ and the relaxed integer variables $y_i \in [0, 1]$.
\end{definition}

\begin{theorem}
Maximum a posteriori estimation for the Gaussian mixture model \eqref{eqn:gen-model} with known covariance is a biconvex mixed-integer nonlinear programming optimization problem.
\end{theorem}
\begin{proof}
The structure of the proof is to first relax the binary variables such that $y_i \in [0,1]$.
Then show the feasible set $\mathcal{B}$ is biconvex.
Finally show the objective function is biconvex.

The relaxed maximum a posterior estimation problem is 
\small
\begin{mini*}|s|
  {\vec{z}, \vec{\mu}, \vec{\pi}}
  {\eta \sum_{i=1}^{n} \sum_{k=1}^K z_{ik} (y_i - \mu_k)^2  - \sum_{i=1}^n \sum_{k=1}^K z_{ik} \log \pi_k }
  {\label{eq:relaxed-map-opt}}
  {}
  \addConstraint{\sum_{k=1}^K \pi_k}{=1}
  \addConstraint{\sum_{k=1}^K z_{ik}}{=1,\quad}{i=1,\ldots,n}
  \addConstraint{-M_k^L \leq \mu_k}{\leq M_k^U,\quad}{k=1,\ldots,K}
  \addConstraint{\pi_k}{\geq 0,\quad}{k=1,\ldots,K}
  \addConstraint{z_{ik}}{\in [0,1],\quad}{i=1, \ldots, n}
  \addConstraint{}{}{k=1, \ldots, K.}
\end{mini*}
\normalsize
For the remainder we consider range of the objective function to be the extended reals $\mathbb{R} \cup \{-\infty, \infty\}$ to simplify the case when $\pi_k=0$ for any $k$.

\subparagraph{Feasible Set is Biconvex}
It is easy to see that the feasible set is convex in $(\bm{\pi}, \bm{\mu})$ for fixed $\bm{z}$ because $\bm{\pi}$ is in the $K$-dimensional simplex and the constraints on $\bm{\mu}$ are box constraints.
Likewise, the constraints on $\bm{z}$ are trivially convex for fixed $(\bm{\pi}, \bm{\mu})$.

\subparagraph{Objective Function is Biconvex}
As in Problem~\eqref{eq:map-opt}, we consider $\eta$ to be known.
For fixed $\bm{z}$, the objective function is the linear combination of a quadratic term in $\bm{\mu}$ and a convex $-\log$ term in $\bm{\pi}$.
Therefore the objective function is convex in $(\bm{\pi}, \bm{\mu})$ for fixed $\bm{z}$.
For fixed $(\bm{\pi}, \bm{\mu})$, the objective function is linear in $\bm{z}$ and therefore convex.
Therefore, the objective function is biconvex.
\end{proof}

\section{Branch-and-Bound Algorithm}
\label{sec:branching}

The general branch-and-bound algorithm is outlined in Algorithm~\ref{alg:bnb}.
\begin{algorithm}
  \footnotesize
  \begin{algorithmic}
  \STATE Initialize the candidate priority queue to consist of the relaxed \gls{minlp} and set $\text{UBD}=\infty$ and $\text{GLBD}=-\infty$
  \WHILE{candidate queue is not empty}
    \STATE Pop the first node off of the priority queue, solve the subproblem, and store the optimal value $f^*$.
    
    \IF{$\vec{z}$ integral}

      \IF{$f^* \leq \textrm{UBD}$}
        \STATE Update $\text{UBD} \leftarrow f^*$
        \STATE Remove node $j$ from candidate queue where $\textrm{LBD}_j > \textrm{UBD}\ \forall j$
    \ENDIF
    \ELSE
      \STATE Select a branching variable $z_{ik}$ according to branching strategy
      
      \STATE Push node for candidate relaxed subproblem $j$ on queue adding constraint $z_{ik}=0$ and set $\textrm{LBD}_j = f^*$

      \STATE Push node for candidate relaxed subproblem $j^\prime$ on queue adding constraint $z_{ik}=1$ and set $\textrm{LBD}_{j^\prime} = f^*$

    \ENDIF

    \STATE Set $\text{GLBD} = \min_j \text{LBD}_j$ for all $j$ in candidate queue
    \ENDWHILE
  \end{algorithmic}
  \caption{Branch-and-Bound}
  \label{alg:bnb}
\end{algorithm}

\subsection{Branching Strategies}
Though the choice of branching strategy should be tailored to the problem, there are three popular general strategies \cite{achterberg2005branching}: \textit{Most Infeasible Branching} -- choose the integer variable with the fraction part closest to 0.5 in the relaxed subproblem solution, \textit{Pseudo Cost Branching} --- choose the variable that provided the greatest reduction in the objective when it was previously branched on in the tree, and \textit{Strong Branching} --- test a subset of the variables at both of their bounds and branch on the one that provides the greatest reduction in the objective function.

The strong branching strategy can be computationally expensive because two optimization problems must be solved for each candidate, so loose approximations such as linear programs are often used.
The pseudo cost strategy does not provide much benefit early in the iterations because there is not much branching history to draw upon.
Rather than using a general-purpose strategy, we have developed a branching strategy more tuned for the finite mixture model problem.

\paragraph{Most Integral Branching}
We implemented a \textit{most integral} branching strategy for the \gls{gmm} \gls{map} problem.
The idea is to first find a solution to a relaxed problem where $z_{ik} \in [0,1]$.
Then, identify those $z_{ik}$ variables that are closest to 1 --- that is, variables that are most definitively assigned to one component.\footnote{Recall that $z_{ik} = 0$ only constrains data point $i$ to be \textit{not assigned to component $k$}, but which of the other $K-1$ components it is assigned to \textit{is not fixed}.}
Those $z_{ik}$ variables are then chosen for branching with the expectation that the relaxed subproblem branch with the constraint that $z_{ik} = 0$ will tend to have a lower bound that is greater than the best upper bound and the subtree can be fathomed.

\subsection{Algorithm Strategies}

The general branch-and-bound algorithm (\Cref{alg:bnb}) is implemented as a tree with nodes indexed in a priority queue.
A node in the candidate priority queue contains a relaxed subproblem and a lower bound for the subproblem.
%
Each branching leads to a child node subproblem. If a lower bound of a child node is greater than the best current upper bound, then the subtree is fathomed and the search space is reduced by a factor of up to one half, drastically improving the computational efficiency.
Given Theorem~\ref{thm:biconvex-minlp} the subproblem at each node is a biconvex nonlinear programming problem.

\clearpage
\section{Comparison of Gene Expression Profiles}
In our experiments 45 samples were classified differently by PAM and our MIQP method.
Here we show how the samples were assigned and plots of the expression profile for the sample and the average of the expression profile of the two assignments.

\begin{table*}[h!]
  \centering
  \caption{Comparison of assignments of samples who are assigned to different subtypes from PAM and MIQP.}
  \begin{tabular}{@{}ccc@{}}
    \toprule
    Sample & PAM Assignment & MIQP Assignment \\
    \midrule
Sample 1 & Normal & Her2 \\
Sample 2 & LumA & Her2 \\
Sample 3 & Basal & Her2 \\
Sample 4 & Normal & Her2 \\
Sample 5 & Normal & Her2 \\
Sample 6 & Normal & Basal \\
Sample 7 & Basal & Her2 \\
Sample 8 & LumB & Her2 \\
Sample 9 & LumB & Her2 \\
Sample 10 & LumB & Her2 \\
Sample 11 & LumB & Her2 \\
Sample 12 & LumB & Her2 \\
Sample 13 & LumB & LumA \\
Sample 14 & LumA & Basal \\
Sample 15 & LumA & Her2 \\
Sample 16 & LumA & LumB \\
Sample 17 & LumA & Basal \\
Sample 18 & LumA & LumB \\
Sample 19 & LumB & LumA \\
Sample 20 & LumB & LumA \\
Sample 21 & Her2 & LumA \\
Sample 22 & Basal & LumA \\
Sample 23 & LumA & LumB \\
Sample 24 & Normal & LumA \\
Sample 25 & Normal & LumA \\
Sample 26 & LumA & LumB \\
Sample 27 & Basal & LumB \\
Sample 28 & Her2 & LumA \\
Sample 29 & Basal & LumB \\
Sample 30 & LumA & LumB \\
Sample 31 & Basal & LumB \\
Sample 32 & LumA & LumB \\
Sample 33 & LumA & LumB \\
Sample 34 & LumA & LumB \\
Sample 35 & LumA & LumB \\
Sample 36 & Normal & Basal \\
Sample 37 & LumA & Her2 \\
Sample 38 & LumB & LumA \\
Sample 39 & LumA & Basal \\
Sample 40 & LumA & Her2 \\
Sample 41 & LumA & LumB \\
Sample 42 & Normal & Her2 \\
Sample 43 & Normal & Her2 \\
Sample 44 & Normal & Her2 \\
Sample 45 & Basal & Her2 \\
    \bottomrule
  \end{tabular}
  \label{tbl:comparison_pam}
\end{table*}

\begin{figure}[h!]
  \centering
  \small
  \includegraphics[scale = 0.65]{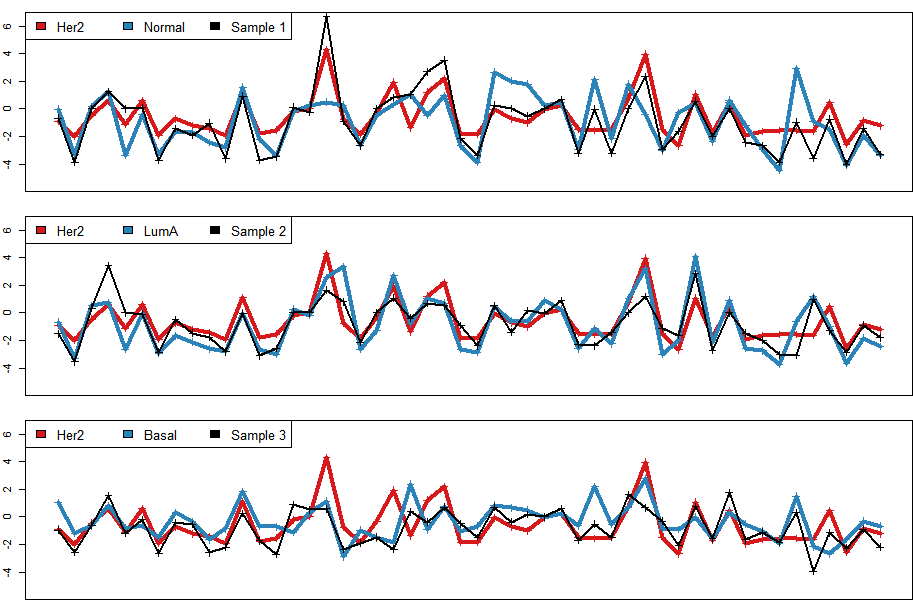}
  \par
  \includegraphics[scale = 0.65]{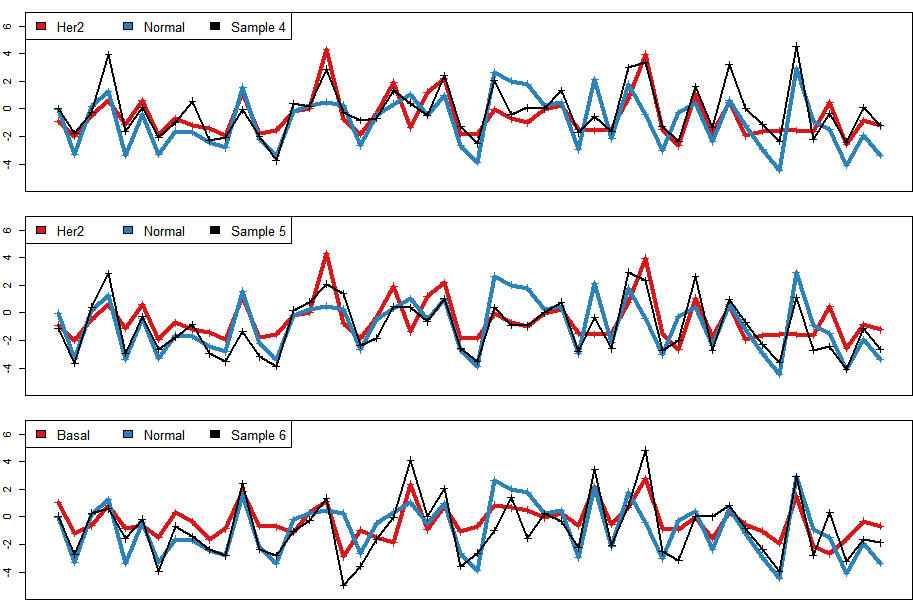}
\end{figure}

\begin{figure}[h!]
  \centering
  \small
  \includegraphics[scale = 0.65]{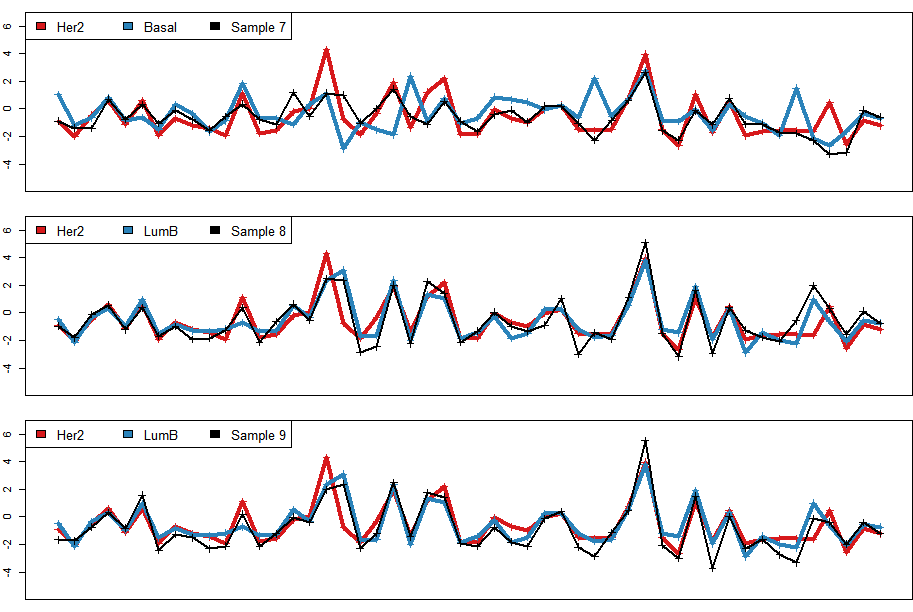}
  \par
  \includegraphics[scale = 0.65]{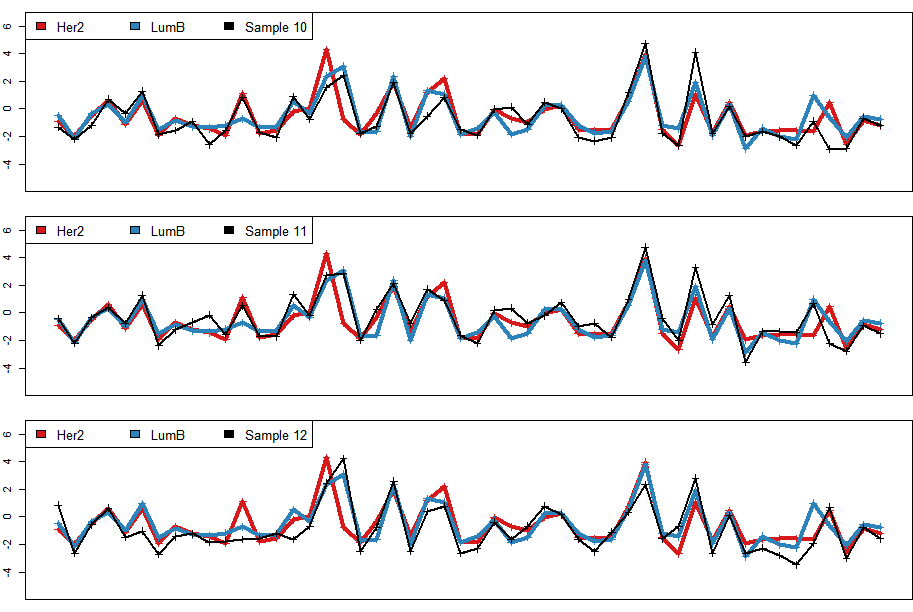}
\end{figure}

\begin{figure}[h!]
  \centering
  \small
  \includegraphics[scale = 0.65]{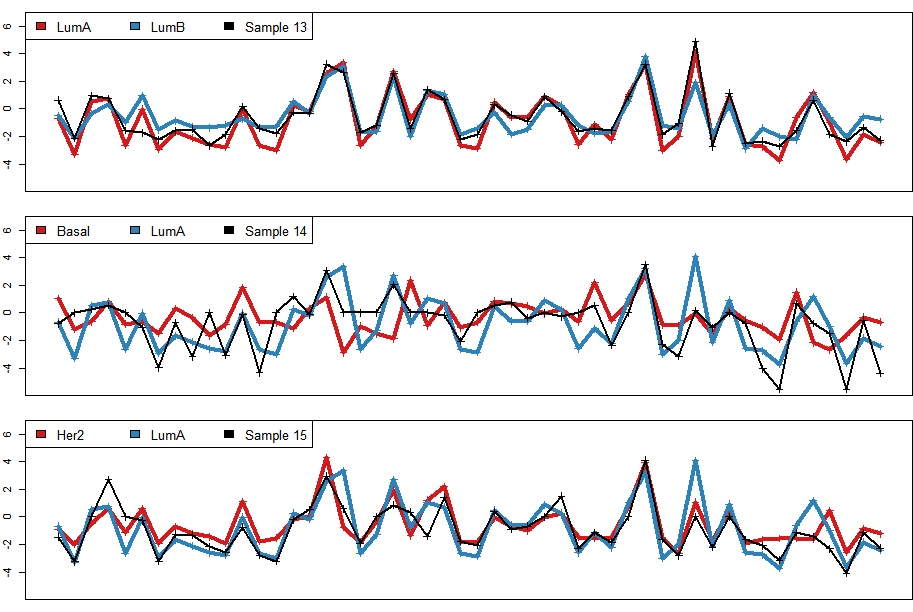}
  \par
  \includegraphics[scale = 0.65]{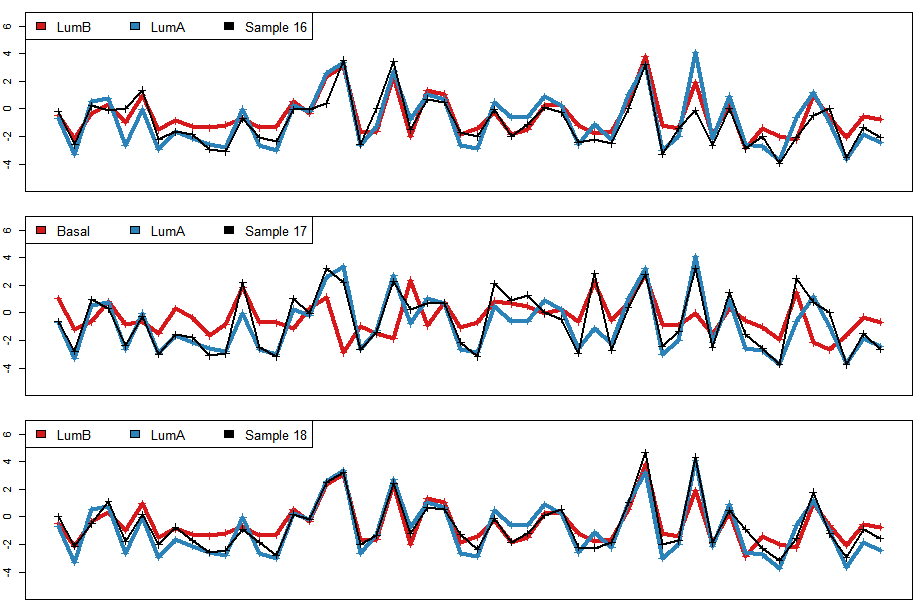}
\end{figure}

\begin{figure}[h!]
  \centering
  \small
  \includegraphics[scale = 0.65]{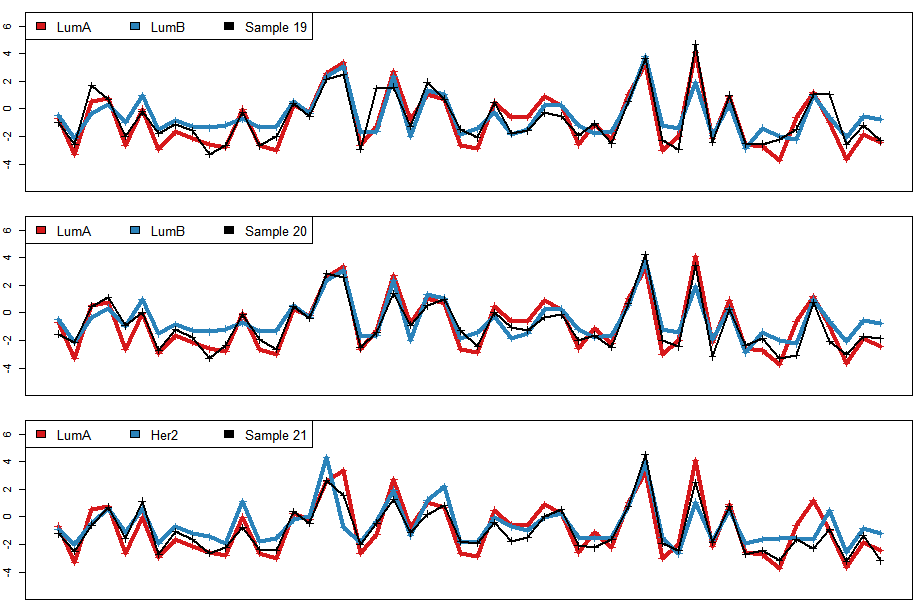}
  \par
  \includegraphics[scale = 0.65]{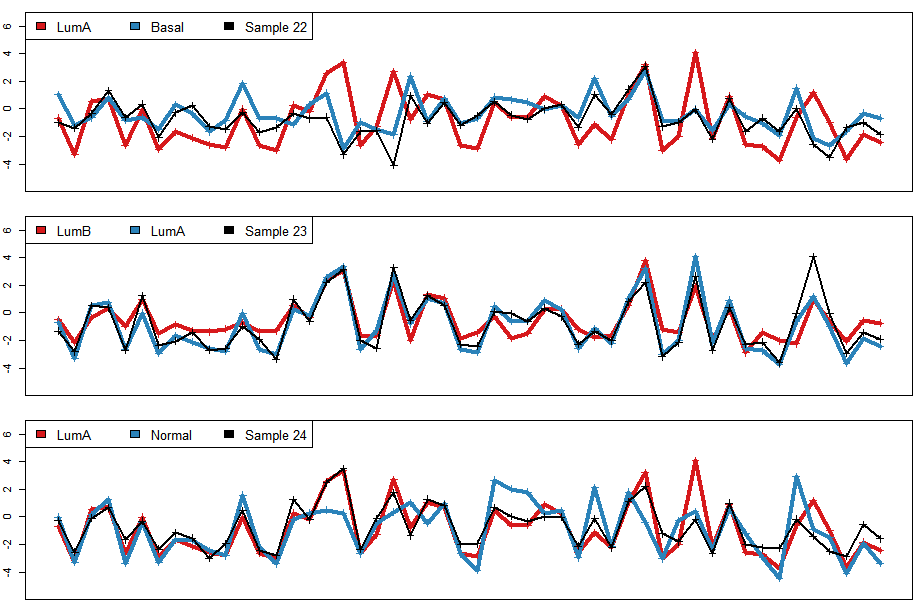}
\end{figure}

\begin{figure}[h!]
  \centering
  \small
  \includegraphics[scale = 0.65]{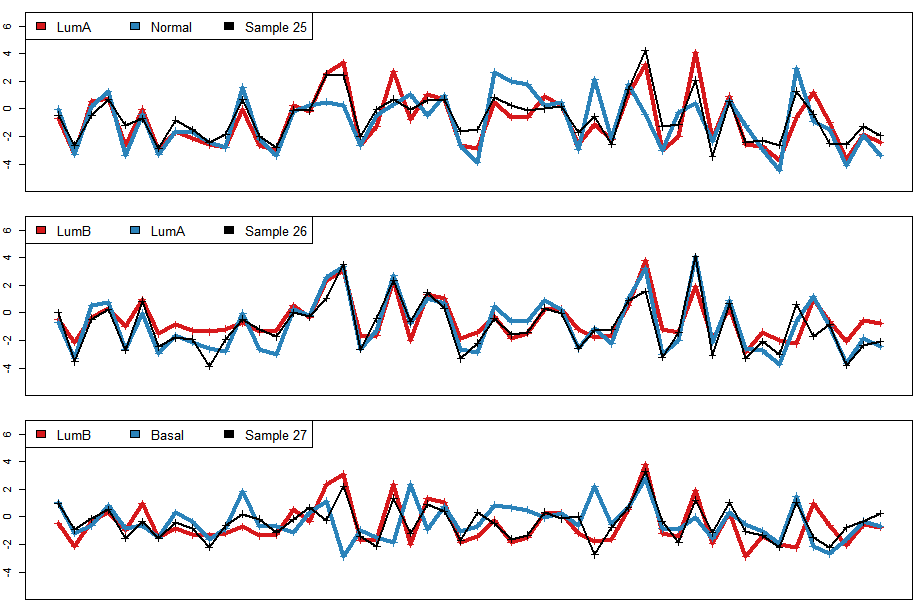}
  \par
  \includegraphics[scale = 0.65]{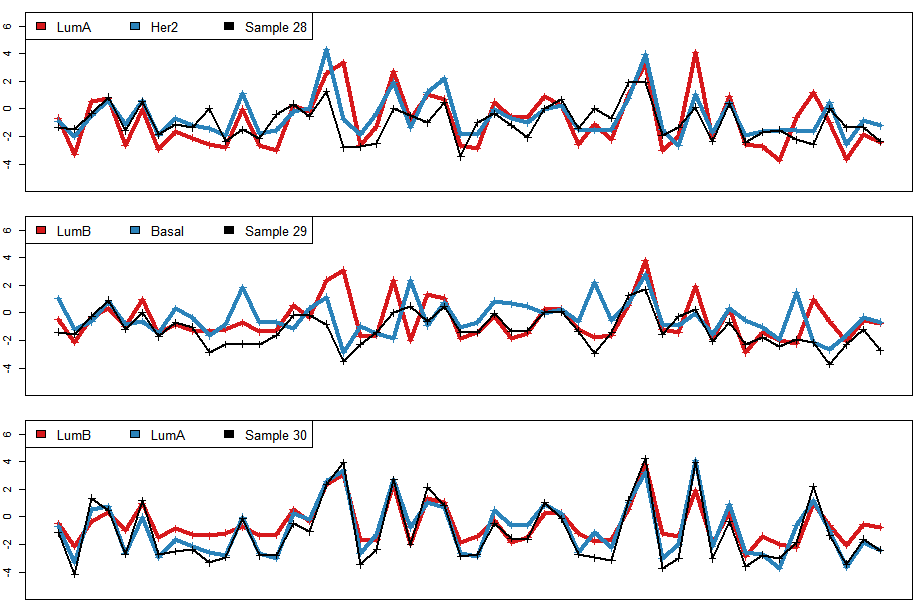}
\end{figure}

\begin{figure}[h!]
  \centering
  \small
  \includegraphics[scale = 0.65]{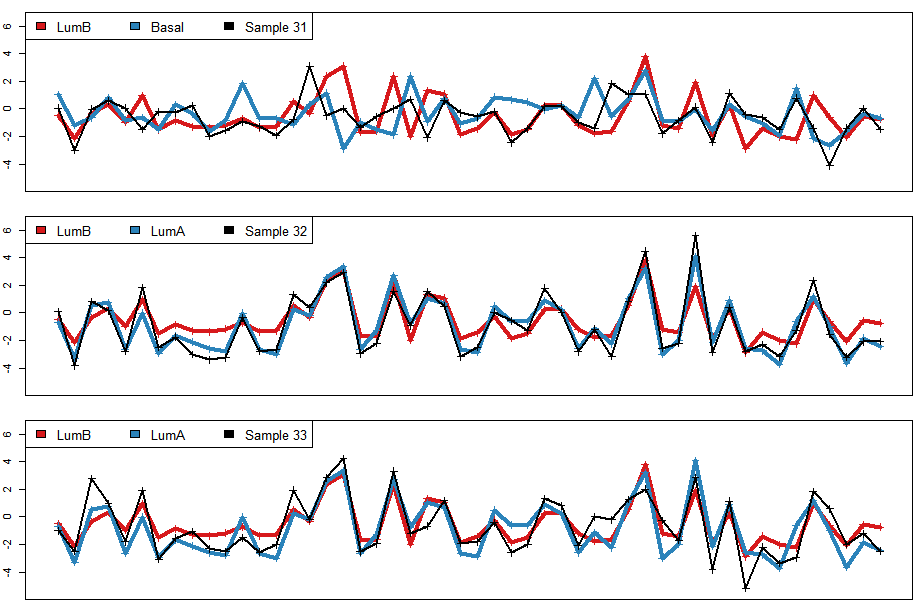}
  \par
  \includegraphics[scale = 0.65]{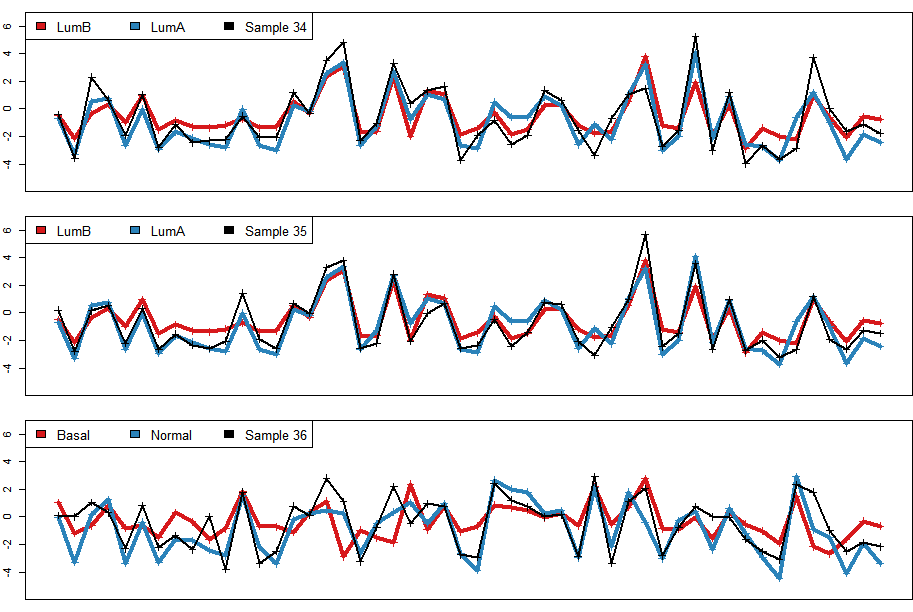}
\end{figure}

\begin{figure}[h!]
  \centering
  \small
  \includegraphics[scale = 0.65]{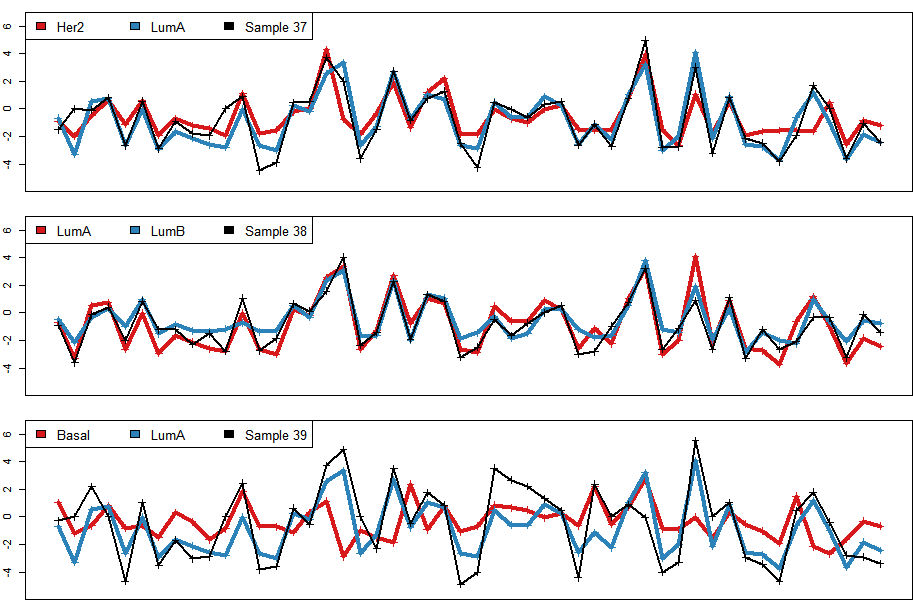}
  \par
  \includegraphics[scale = 0.65]{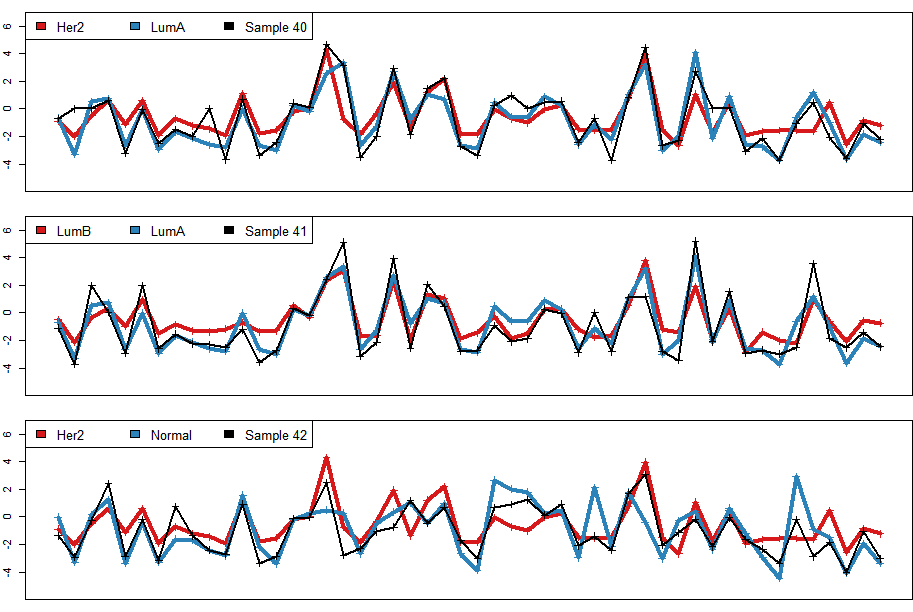}
\end{figure}

\begin{figure}[!ht]
  \centering
  \small
  \includegraphics[scale = 0.65]{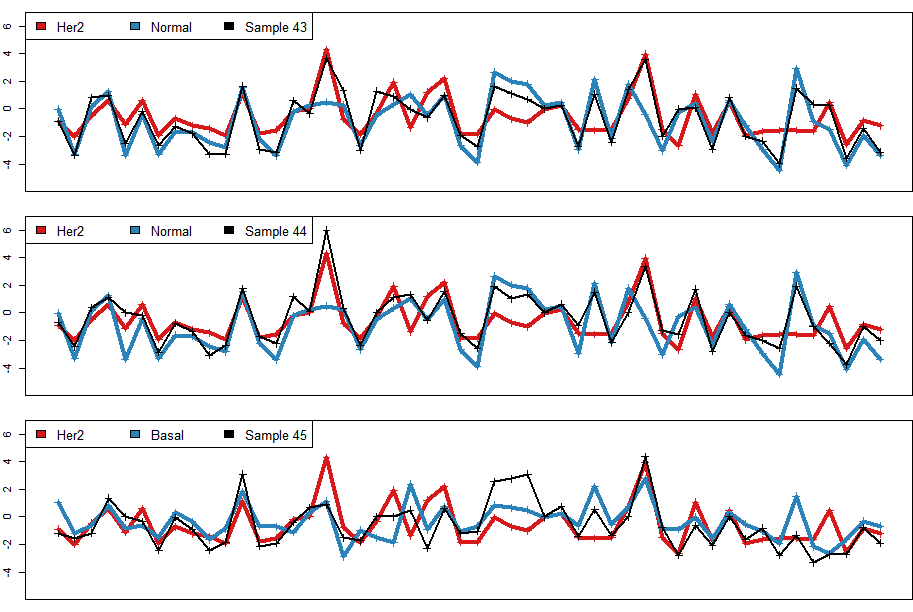}
  \caption{Comparison of gene expression profiles of samples that are assigned to different subtypes from PAM and MIQP. The x-axis is the genes in the PAM50 signature shown in Figure~\ref{fig:expression_comparison}.}
\end{figure}
\end{document}